\documentclass[letterpaper, 10 pt, conference]{ieeeconf}  

\usepackage{xcolor}
\usepackage{amsmath}
\usepackage{amsthm}
\usepackage{amsfonts,amssymb,mathtools}
\usepackage{algorithm}
\usepackage{algpseudocode}
\usepackage{cleveref}

\newcommand{\KK}{\mathcal{K}}

\newcommand{\XX}{\mathcal{X}}

\newcommand{\real}{\mathbb{R}}

\newcommand{\ie}{\textit{i.e. }}
\newcommand{\st}{\textit{s.t. }}

\crefname{asp}{assumption}{assumptions}
\Crefname{asp}{Assumption}{Assumptions}
\crefname{lem}{lemma}{lemmas}
\Crefname{lem}{Lemma}{Lemmas}

\theoremstyle{definition}
\newtheorem{thm}{Theorem}
\newtheorem{lem}[thm]{Lemma}

\newtheorem*{cor}{Corollary}

\theoremstyle{definition}
\newtheorem{defn}{Definition}

\newtheorem{exmp}{Example}

\theoremstyle{remark}
\newtheorem*{rem}{Remark}

\theoremstyle{remark}

\newcommand{\rre}[1]{{\color{black} #1}}

\usepackage[utf8]{inputenc}
\usepackage{comment}
\usepackage{subfigure}
\usepackage{graphicx}
\usepackage[normalem]{ulem}
\usepackage{booktabs}

\usepackage{cleveref}

\IEEEoverridecommandlockouts                              

\title{\LARGE \bf
Robust Safe Control with Multi-Modal Uncertainty}

\author{Tianhao Wei$^{1}$, Liqian Ma$^{2}$, Ravi Pandya$^{1}$, and Changliu Liu$^{1}$
\thanks{*This material is based upon work supported by the National Science Foundation under Grant No. 2144489.}
\thanks{$^{1}$These authors are with Robotics Institute, Carnegie Mellon University, {\tt\small twei2, rapandya, cliu6@andrew.cmu.edu}}%
\thanks{$^{2}$Liqian Ma is with the Department of Mechanical Engineering, Tsinghua University. Work done during an internship at Carnegie Mellon. {\tt\small mlq19@mails.tsinghua.edu.cn}}
}

\begin{document}

\maketitle
\thispagestyle{empty}
\pagestyle{empty}

\begin{abstract}
Safety in dynamic systems with prevalent uncertainties is crucial. Current robust safe controllers, designed primarily for uni-modal uncertainties, may be either overly conservative or unsafe when handling multi-modal uncertainties. To address the problem, we introduce a novel framework for robust safe control, tailored to accommodate multi-modal Gaussian dynamics uncertainties and control limits.

We first present an innovative method for deriving the least conservative robust safe control under additive multi-modal uncertainties. Next, we propose a strategy to identify a locally least-conservative robust safe control under multiplicative uncertainties. Following these, we introduce a unique safety index synthesis method. This provides the foundation for a robust safe controller that ensures a high probability of realizability under control limits and multi-modal uncertainties.

Experiments on a simulated Segway validate our approach, showing consistent realizability and less conservatism than controllers designed using uni-modal uncertainty methods. The framework offers significant potential for enhancing safety and performance in robotic applications.
\end{abstract}

\section{Introduction}

Robotic systems prioritize safety as a critical element. Safe control, serving as the system's final defense, ensures real-time safety by maintaining the system state within a safety set, a concept known as forward invariance~\cite{liu2014control}.
Energy-function-based methods have been proposed to facilitate safe control~\cite{wei2019safe}. These methods employ the energy function, also referred to as the safety index, barrier function, or Lyapunov function, to convert the safe control problem into online quadratic programming (QP). Both deterministic~\cite{wei2022safe, zhao2023safety, wei2023zero} and uncertain~\cite{chen2022safe, zhao2023probabilistic, wei2022persistently} dynamic models have been extensively explored in QP-based safe control.

However, existing works on robust safe control predominantly assume uni-modal uncertainties~\cite{garg2021robust, jankovic2018robust, grover2022control, castaneda2021pointwise}. In practical applications, multi-modal uncertainty is common. For instance, a car at an intersection can either go straight, turn left, or turn right, instead of moving in arbitrary directions. Applying uni-modal methods to tackle multi-modal uncertainty can lead to excessively conservative or dangerous behaviors. Furthermore, the conservativeness of the uni-modal approach can result in unrealizable control in the presence of control limits. It remains challenging to design a robust safe controller under multi-modal uncertainties that is both persistently feasible and non-conservative.

The major challenge of dealing with multi-modal uncertainties is how to design the control-space safety constraint for each mode, such that the control is least conservative and the probability of safety is high.
Consider a scenario where a pedestrian tries to avoid an oncoming car with two modes of uncertainty: ``turn left" and ``turn right". Including both modes can paralyze the pedestrian's decision-making. However, if we ascertain that the car has a $99\%$ probability of turning left, it is safe to ignore the ``turn right" mode while still ensuring a high probability of safety. Additionally, synthesizing a safety index that guarantees feasible control across most states becomes challenging with multi-modal uncertainty, as many assumptions made in previous methods, like the continuity of the dynamic model~\cite{zhao2023probabilistic} or Gaussian process uncertainty~\cite{wei2022persistently}, no longer hold.

\begin{figure}
    \centering
    \includegraphics[width=\linewidth]{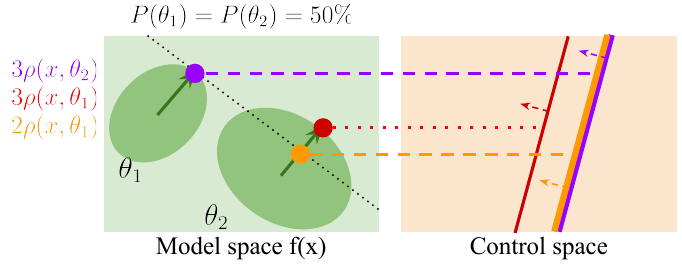}
    \caption{Illustration of a dynamic model under additive uncertainty and the corresponding safe control sets. Each possible dynamic model corresponds to a linear constraint in the control space. The green arrows show the gradient direction of the safety index ($\nabla\phi$). The black dotted line denotes the contour line of the Lie derivative ($\nabla\phi\cdot f$). The arrows in the control space show the corresponding safe half space.  When the desired safe probability is $97.5\%$ and \textcolor[RGB]{153,0,255}{$k_1=3$}, we can relax \textcolor[RGB]{255,153,0}{$k_2=2$}, which gives a looser constraint than \textcolor[RGB]{204,0,0}{$k_2=3$}. Therefore we can have a larger feasible control region.}
    \label{fig:additive}
\end{figure}
Our approach comprises three key components. Firstly, in the additive uncertainty case (``additive" vs ``multiplicative" refers to how the uncertainty enters into the system dynamics with respect to the control input), we demonstrate how to solve the multi-modal robust safe control problem by formulating an optimization with two constraints, one constraint on the overall safe probability and one constraint on control safety derived from all modes. We show that the constraints are scalarly monotonic, hence the least conservative safe control can be found by binary search. Secondly, In the multiplicative case, the constraints are no longer monotonic, therefore we propose a bi-level optimization to find a locally least conservative safe control. The upper level, optimizes the confidence level of each mode with numerical methods, while the lower level, which is a second-order cone program (SOCP), finds the locally least conservative safe control given the confidence levels. Lastly, we provide a method to offer a probabilistic guarantee of feasibility for a given safety index and illustrate how to synthesize a persistently feasible safety index.

Our method significantly tightens the bounds compared to naive methods that directly overapproximate multi-modal uncertainty as uni-modal uncertainty. Our method maximizes the feasible control region under additive multi-modal uncertainty and enlarges the feasible control region even under multiplicative uncertainty. Our method also ensures a probabilistic guarantee of the persistent feasibility of the safety index under multi-modal uncertainty.

The remainder of the paper unfolds as follows: \Cref{sec: formulation} lays out the problem and introduces the necessary notations. \Cref{sec: method} discusses the proposed multi-modal robust safe control framework. \Cref{sec: exp} evaluates the proposed methods on a Segway robot.

\section{Problem Formulation}\label{sec: formulation}

We consider the following nonlinear control-affine dynamic system with multi-modal uncertainty:\footnote{Although our method assumes control affine dynamics, it is applicable to non-control affine systems since we can always have a control affine form through dynamics extension~\cite{liu2016algorithmic}.}
\begin{align}
    \dot x = f(x,\theta) + g(x, \theta) u, \label{eq: dyn}
\end{align}
where state $x \in X \subseteq \real^{n}$ and control $u \in U \subseteq \real^{m}$. $\theta$ follows a known discrete probability distribution over $\{\theta_1, \ldots, \theta_n \}$. 
The terms $f(x,\theta) \in \real^{n}$ and $g(x,\theta) \in \real^{n\times m}$ are random matrices that follow multi-modal normal distributions, that is, $f(x,\theta_i)\sim \mathcal{N}(\mu_f(x,\theta_i), \Sigma_f(x,\theta_i))$, $g(x,\theta_i)\sim \mathcal{N}(\mu_g(x,\theta_i), \Sigma_g(x,\theta_i))$.

We consider the safety specification as a requirement that the system state should be constrained in a closed and connected set $\XX_S \subseteq X$, which we call the safe set. We assume $\XX_S$ is the zero-sublevel set of a scalar safety measure $\phi_0: X \mapsto \real$ given by the user. That is $\XX_S = \{x \mid \phi_0(x) \leq 0, x \in X\}$.
Constraining states inside $\XX_S$ can be expressed as a \textit{forward invariance} problem: when $x(t_0) \in \XX_S$, ensure $x(t) \in \XX_S,\ \forall t>t_0$. Forward invariance can be guaranteed with minimal invasion by QP-based safe set algorithms.
\begin{defn}[QP-based safe set algorithms]\label{algo: ssa}
Given a reference control $u_{\text{ref}}$ and a safety index $\phi: X \mapsto \real$, QP-based safe set algorithms find safe control $u$ by: 
\begin{align}
    &\min_{u\in U} \|u - u_{\text{ref}}\|^2\ \\
    \st &\dot \phi(x,u) \leq -\gamma(\phi(x)),\label{eq: lie}
\end{align}
where 
$\dot \phi(x,u) := \nabla \phi(x) \dot x$, $\gamma$ is a piecewise smooth function and $\gamma(\rre{\phi(x)}) > 0$ when $\phi(x) > 0$. $\gamma$ can be non-continuous and designed differently.
A special case of $\gamma$ is an extended class $\KK$ function on $\phi(x)$, corresponding to the control barrier function (CBF) method~\cite{wei2019safe}. 
Safe set algorithms require $\phi$ to be \textit{persistent feasible}: $\forall x, \exists u$, such that $\dot \phi(x,u) \leq -\gamma(\rre{\phi(x)})$. However, a user-defined safety index $\phi_0$ may not be naturally persistently feasible. It is often necessary to design a new safety index $\phi$ based on $\phi_0$ so that we can enforce forward invariance in $S: \{\phi \leq 0\} \subseteq \XX_S$~\cite{liu2014control}. 
\end{defn}

Extending \cref{algo: ssa} to probabilistic uncertainty requires the safety constraint to hold almost surely. We define the extended problem as \textit{Chance Constrained Safe Control}.
\begin{defn} (Chance constrained safe control) For a small $\epsilon_f$, we define the following problem:
\begin{subequations}
\begin{align}
    &\min_{u\in U}  \|u - u_{\text{ref}}\|^2 \\  
    &\st P\left(\dot\phi(x,u) \leq -\gamma(\phi(x))\right) \geq 1-\epsilon_f.
\end{align}\label{eq: prob_con}
\end{subequations}
\end{defn}
In the multi-modal uncertainty case, \eqref{eq: prob_con} equals to:
\begin{subequations}
\begin{align}
    & \min_{u\in U}  \|u - u_{\text{ref}}\|^2 \\  
    & \st \sum_i P(\theta_i) p_i \geq 1-\epsilon_f. 
\end{align}\label{eq: mmp-qp}
\end{subequations}
where 
\begin{align}
    p_i := P\left(\dot\phi(x,u) \leq -\gamma(\phi) \mid \theta_i \right)
\end{align}
are design variables. We can increase $p_i$ by considering a wider range of uncertainty for mode $i$, but a wider range of uncertainty may lead to more conservative control.  

We say $\phi$ is a \textit{probabilistically robust safety index} (PR-SI) if it ensures feasibility with a high probability.
\begin{defn}[PR-SI]\label{def:PR-SI}
A safety index $\phi$ is a PR-SI if there exists a piecewise smooth, strictly increasing function $\gamma$, and $\gamma(0)=0$, such that for $(1-\epsilon_s)\times 100\%$ states, the \textit{feasibility} condition holds with a high probability ($1-\epsilon_f$). 

We define a random event A: given a state $x$,  a control $u$ is a feasible control, that is $\dot \phi(x,u) \leq -\gamma(\phi(x))$. For convenience, we may say a state is \textit{feasible} if the feasibility condition holds with high probability, that is $P(A) > 1-\epsilon_f$. We define another random event B: More than $(1-\epsilon_s)\times 100\%$ states are feasible. A PR-SI ensures that $P(B) > 1-\epsilon_c$. For example, we want to ensure that 
$P(\text{more than } 99.99\% \text{ states are feasible}) > 99.99\%$.
\end{defn}

This paper first aims to develop a method to obtain such a $\phi$ that satisfies \cref{def:PR-SI}. Then we aim to develop a method to solve \eqref{eq: mmp-qp} efficiently and non-conservatively. A naive approach is modeling all uncertainties with a single Gaussian, such that there will be only one mode and one constraint. Then uni-modal algorithms can be applied~\cite{wei2022persistently}. However, this approach is conservative because modeling all uncertainties with a single Gaussian inevitably leads to more conservative control as shown in \cref{fig:exp-add}. In this work, we propose a non-conservative method that explicitly considers the multi-modal uncertainty.

\section{Multi-Modal Robust Safe Set Algorithm} \label{sec: method}

This section first presents an efficient and tight solver of \eqref{eq: mmp-qp} under multi-modal uncertainty, which gives an optimal solution of \eqref{eq: mmp-qp} for additive uncertainties and a locally optimal solution for multiplicative uncertainties. We then present a safety index synthesis method that improves feasibility for arbitrary dynamic models and arbitrary uncertainties.

\subsection{Multi-Modal Additive Uncertainty}\label{sec: add}

We first consider the case where only $f(x)$ has uncertainty. That is, $\forall x, \mu_g(x,\theta_i) = 0, \Sigma_g(x, \theta_i) = 0.$ It is difficult to directly optimize $p_i$ to balance safety and conservativeness. Therefore, we introduce a variable $k_i$. $f(x, \theta_i)$ has the probability of $p_i$ to fall within $k_i \sigma$ ($k_i$ standard deviation) bound around the mean $\mu_f(x,\theta_i)$. Then we show how to design $p_i$ with $k_i$. 

The safety constraint can be transformed as follows:
\begin{align}
    & \nabla \phi(x) [f(x,\theta) + g(x) u] \leq -\gamma(\phi(x))\\
    & \nabla \phi(x) g(x) u \leq -\gamma(\phi(x)) - \nabla \phi(x) f(x,\theta).
\end{align}

Because we assume Gaussian uncertainty, we consider the confidence bound for each mode to be an ellipsoid. Given the ellipsoid bound of $f(x,\theta_i)$, the safety constraint in the control space is determined by the projection of the ellipsoid to $\nabla\phi$. Let $\rho(x,\theta)$ denote the worst-case direction in the $1\sigma$ ellipsoid, \ie
\begin{align}
    \rho(x,\theta) := \max_{\delta^T\Sigma_f^{-1}(x,\theta)\delta \leq 1} \nabla\phi \cdot \delta
\end{align}
where $\delta = f(x,\theta) - \mu_f(x,\theta)$. The $k_i\sigma$ worst case situation is $k_i\rho(x,\theta_i)$. Then the safety constraint for mode $i$ that considers $k_i\sigma$ range of uncertainty can be written as
\begin{align}
\nabla\phi g(x) u\leq -\gamma(\phi) - \nabla\phi \mu_f(x,\theta_i) - k_i \rho(x,\theta_i), \forall i.
\end{align}

Note that $k_i = \sqrt{\chi^2_1 (p_i)}$, where $\chi^2_n(p)$ represents the quantile function of the Chi-squared distribution with $n$ degrees of freedom. As a practical example, for $p_i = 99.73\%$, $k_i=3$, which is the well-established $3 \sigma$ rule in statistics. We can let $k_i=3$ for all modes, in which case, the probability $P(A)$ is guaranteed to be larger than $99.73\%$. However, such assigning may result in overly conservative control. As shown in \cref{fig:additive}, each possible dynamic model corresponds to a safety constraint in the control space. The safety constraint only depends on the most conservative one. If the desired safe probability is $97.5\%$,  whenever $k_1=3$, we can let $k_2=2$, which gives a looser constraint and still keeps a high probability of safety.

With $p_i(k_i) = {\chi^2_1}^{-1}(k_i^2)$, the multi-modal safe control problem \eqref{eq: mmp-qp} can be formulated as follows
\begin{subequations}
\begin{align}
    & \min_{u\in U, k_1\cdots, k_n} \|u - u_{\text{ref}}\|^2 
 \label{eq:bilevel_obj}\\  
    \st &  \sum_i P(\theta_i) p_i(k_i) \geq 1-\epsilon_f ,\label{eq:bilevel_prob}\\
    & \nabla\phi g(x) u \leq \min_{i} \{ -\gamma(\phi) + o(x,\theta_i,k_i)\}\label{eq:bilevel_safe}.
\end{align}\label{eq:bilevel}
\end{subequations}
where $o(x,\theta_i,k_i) := -\nabla\phi \mu_f(x,\theta_i) - k_i\rho(x,\theta_i)$. Easy to see that the safety constraint only depends on the lowest $o(x,\theta_i,k_i)$. The least conservative safe control is obtained when the RHS of \eqref{eq:bilevel_safe} is maximized with a set of $k_i$ that satisfies \eqref{eq:bilevel_prob}. 


\begin{algorithm}
\caption{Binary search for optimal additive uncertainty}\label{algo: binary}
\begin{algorithmic}
\State $l_{k1} \gets 0, r_{k1} \gets 10$
\While{$r_{k1} - l_{k1} \geq \epsilon_0$}
    \State $k_1 = (l_{k1} + r_{k1})/2$
    \State Compute $o(x,\theta_1,k_1).$
    \State Solve $k_i$ for all $i$ based on $o(x,\theta_1,k_1) = o(x,\theta_i,k_i)$.
    \If{$\sum_i P(\theta_i)p_i > 1-\epsilon_f$} 
        \State $r_{k1} = k_1$ \Comment{Can have smaller $k_1$}
    \Else
        \State $l_{k1} = k_1$ \Comment{Need larger $k_1$}
    \EndIf
\EndWhile
\State Return $o(x,\theta_i,k_i)$ computed with $k_1 = r_{k1}$.
\end{algorithmic}
\end{algorithm}



\begin{lem}\label{lem:sufficient}
    The RHS of \eqref{eq:bilevel_safe} is maximized $\iff$
    \begin{align}
    & \sum_i P(\theta_i)p_i(k_i) = 1-\epsilon_f \label{eq: P_equal}\\
    & o(x,\theta_1,k_1) = o(x,\theta_i,k_i), \forall i \label{eq: o_equal}
    \end{align}

    \begin{proof}
        1) RHS is maximized $\implies$ \eqref{eq: P_equal}: $p_i$ is a monotonically increasing function of $k_i$. Therefore if $\sum_i P(\theta_i)p_i > 1-\epsilon_f$, we can decrease all $k_i$ simultaneously, which guarantees $\min_i o(x,\theta_i, k_i)$ increase.
        
        2) RHS is maximized $\implies$ \eqref{eq: o_equal}: Suppose $\min_i o(x,\theta_i,k_i) < \max_j o(x,\theta_j,k_j)$, then we can decrease $k_i$ and  increase $k_j$ such that $\min_i o(x,\theta_i,k_i)$ increases while keeping $\sum_i P(\theta_i)p_i(k_i)$ unchanged.
        
        3) RHS is maximized $\impliedby$ \eqref{eq: P_equal}, \eqref{eq: o_equal}. We prove by contradiction. Suppose we denote the global optima by $\{k_i^*\}$ and there exists a local optima denoted by $\{\hat k_i\}$. Such that $o(x,\theta_i,k_i^*) > o(x,\theta_i,\hat k_i)$ and they both satisfy \eqref{eq: P_equal} and \eqref{eq: o_equal}.  Because $o(x,\theta_i,k_i)$, $p_i(k_i)$ are monotonic functions of $k_i$.
        We have for all $i$, $o(x,\theta_i,k_i^*) > o(x,\theta_i,\hat k_i) \implies k_i^* < \hat k_i \implies p_i^* < \hat p_i$. Therefore $P(\theta_i)p_i^* < \sum_i P(\theta_i)\hat p_i $, contradicts with the fact that $\sum_i P(\theta_i)p_i^* = 1 - \epsilon_f = \sum_i P(\theta_i)\hat p_i$ at optima.
    \end{proof}
\end{lem}

Given \cref{lem:sufficient}, $k_i$ that maximizes RHS of \eqref{eq:bilevel_safe} and derives the least conservative safe control can be found by \cref{algo: binary}. The proof follows.

\begin{lem}
    A set of $k_i$ that satisfies \eqref{eq: P_equal} and \eqref{eq: o_equal} can be found by binary search on $k_1$, which has a constant time complexity $O(\log(10/\epsilon_0))$, where $\epsilon_0$ is the desired precision of $k_i$.
    \begin{proof}
        Given arbitrary $k_1$, we can enforce the satisfaction of \eqref{eq: o_equal} by directly computing $k_i$ for all $i$ because $o(x,\theta_i,k_i)$ is a linear function of $k_i$. Furthermore, $o(x,\theta_i,k_i)$ is a monotonic decreasing function of $k_i$, and because $p_i(k_i)$ is a monotonic increasing function of $k_i$, $\sum_i P(\theta_i)p_i$ is a monotonic decreasing function of $k_1$.
    \end{proof}
\end{lem}

\begin{figure}
    \centering
    \includegraphics[width=1.\linewidth]{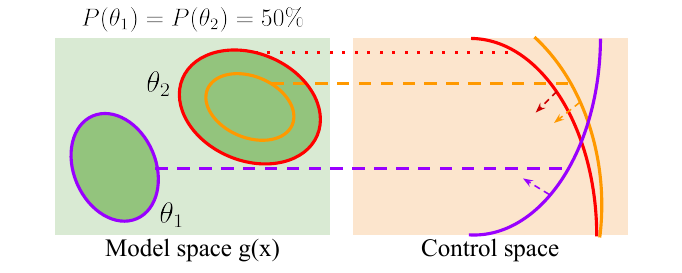}
    \caption{Safety constraints under multiplicative uncertainty. Each dynamic model bound corresponds to a \textbf{nonlinear} constraint in the control space. Therefore, it is difficult to maximize the feasible control set. But we can always find a locally least conservative safe control by numerical methods.}
    \label{fig:multiplicative}
\end{figure}

\subsection{Multi-Modal Multiplicative Uncertainties}\label{sec: mul}
Now we consider the case when $f(x)$ and $g(x)$ both have uncertainties.
In the additive case, we consider the $p_i$ confidence bound of $f(x,\theta_i)$. However, when dealing with multiplicative uncertainties, both $f(x,\theta_i)$ and $g(x,\theta_i)$ are subject to uncertainty. We assume $f(x,\theta_i)$ and $g(x,\theta_i)$ are independent. To ensure the safe constraint can be satisfied with the probability $p_i$, we have to consider a $p_i^f$ confidence bound of $f(x,\theta_i)$ and a $p_i^g$ confidence bound of $g(x,\theta_i)$, such that $p_i^f p_i^g = p_i$. For simplicity, let $p_i^f = p_i^g = \sqrt{p_i}$.

\cite{wei2022persistently} proves that, in the multiplicative Gaussian uncertainty case,  $P\left[\dot\phi(x,u) \leq -\gamma(\phi(x)) \mid \theta=\theta_i \right] > p_i$ can be realized by a second-order cone constraint:
\begin{align}
    ||L_i(p_i)^T u||\leqslant -\mu_i^T u + c_i(p_i)
\end{align}
where $L_i(p_i) L_i(p_i)^T =  \chi_{nm}^{2}(\sqrt{p_i}) \nabla \phi(x)^T \Sigma_g(x,\theta_i) \nabla \phi(x)$ ($L_i(p_i)$ can be solved by Cholesky decomposition), $\mu_i := \nabla \phi(x) \cdot \mu_g(x,\theta_i)$, $c_i(p_i) := -\gamma(\phi(x)) - \nabla\phi(x) \mu_f(x,\theta_i) - \sqrt{\chi^2_{1}(\sqrt{p_i})} \rho(x,\theta_i)$. Then the multi-modal safe control problem can be formulated as
\begin{subequations}
\begin{align}
     \min_{u\in U, p_1,\cdots,p_n} & \|u - u_{\text{ref}}\|^2 
 \label{eq:bilevel_obj_mul}\\  
    \st &  \sum_i P(\theta_i) p_i \geq 1-\epsilon_f ,\label{eq:bilevel_prob_mul}\\
    & ||L_i(p_i)^T u||\leqslant -\mu_i^T u + c_i(p_i),\forall i\label{eq:bilevel_safe_mul}.
\end{align}\label{eq:bilevel_mul}
\end{subequations}
It is difficult to directly find the least conservative constraint as we did in the additive case because the constraint is nonlinear as shown in \cref{fig:multiplicative}. However, we can solve \eqref{eq:bilevel_mul} with a bi-level method. At the upper level, we optimize $p_i$ with numerical methods, while at the lower level, we determine $u$ through second-order cone programming (SOCP). To be specific, we begin by sampling a set of $p_i$ that satisfies \eqref{eq:bilevel_prob_mul}. Subsequently, \eqref{eq:bilevel_mul} reduces to a SOCP, allowing $u$ to be optimally solved. With $u$, we then compute the gradient of \eqref{eq:bilevel_obj_mul} with respect to $p_i$. The gradient information enables finding near-optimal ${p_i}$ with numerical optimization methods.




\subsection{Safety Index Synthesis with Probabilistic Feasibility Guarantee}

We first present theories to provide a probabilistic guarantee of feasibility given a safety index $\phi$, then we show how to optimize the safety index to improve the feasibility.

\begin{lem}[Sampling guarantee]\label{lem:sampling}
    Suppose 1. we uniformly sample $N$ times from $X$; 2. given a $\phi$, $N_f$ samples have feasible safe control and $N_n$ samples have no feasible safe control. Then the percent of feasible states, denoted by $q$, has the distribution as follows for $z \in (0,1)$
    \begin{align}
        P(q=z \mid N_f, N_n) = \frac{z^{N_f+\alpha-1}(1-z)^{N_n+\beta-1}}{\mathrm{~B}(N_f+\alpha, N_n+\beta)}
    \end{align}
    where $\alpha,\beta$ are chosen to reflect any existing belief or information on q ($\alpha =1$ and $\beta =1$ would give a uniform prior distribution). $B(\alpha, \beta)$ is the Beta function acting as a normalising constant:
    \begin{align}
        B(\alpha, \beta) = \frac{(\alpha-1)!(\beta-1)!}{(\alpha+\beta-1)!}.
    \end{align}
\end{lem}
\begin{proof}
    Consider the feasibility of a sample as a random variable. This is a typical Bernoulli process with an unknown probability of success $q \in [0,1]$. This random variable will follow the binomial distribution. The usual conjugate prior~\cite{bernardo2009bayesian} is 
    \begin{align}
        P(q = z) & = \frac{z^{\alpha-1} (1-z)^{\beta-1}}{B(\alpha, \beta)}
    \end{align}
    
    Given $q=z$, the probability of $N_f$ samples having feasible safe control and $N_n$ samples having no feasible safe control is
    \begin{align}
        P(N_f, N_n \mid q = z) & = \left(\begin{array}{c}N_f+N_n \\ N_f\end{array}\right) z^{N_f}(1-z)^{N_n}
    \end{align}

    With $N_f$ and $N_n$ known, the probability of $q=z$, \ie the posterior probability, can be calculated as

    \begin{align}
        P(q=z \mid N_f, N_n) & =\frac{P(N_f, N_n \mid z) P(z)}{\int P(N_f, N_n \mid y) P(y) d y}\\
        & =\frac{z^{N_f+\alpha-1}(1-z)^{N_n+\beta-1}}{\mathrm{~B}(N_f+\alpha, N_n+\beta)}
    \end{align}  
\end{proof}
\begin{cor}
    The cumulative distribution function $P(q < z \mid N_f, N_n)$ is the regularized incomplete Beta function \cite{bernardo2009bayesian}:
    \begin{align}
        P(q<z \mid N_f, N_n) = \frac{B(z; N_f+\alpha, N_n+\beta)}{B(N_f+\alpha, N_n+\beta)}.
    \end{align}
    where
    \begin{align}
        B(z; \alpha, \beta) = \int_0^z y^{\alpha-1}(1-y)^{\beta-1} dy.
    \end{align}
\end{cor}

\begin{exmp}
    Suppose all samples are feasible, that is, $N=N_f=100000$, $N_n=0$, given an unbiased uniform distribution prior ($\alpha=1, \beta=1$), we can derive that $P(\text{more than } 99.99\% \text{ states are feasible}) > 99.99\%$.
\end{exmp}

\begin{rem}
    In reality, some states are visited more frequently than others, such as a stable equilibrium. In this case, the feasibility at these states is more important than other states. To reflect this nature in the safety index, we may adjust the sampling strategy. We can sample trajectories instead of single states using a given nominal or safe controller denoted by $\pi$. Then applying the same analysis as above, we can derive $P(\text{more than } 99.99\% \text{ states \textbf{to be visited by $\pi$} are feasible})$.
\end{rem}


Now we can formulate the safety index synthesis as the following problem
\begin{align*}
    \max_{\eta} P(\text{more than } 99.99\% \text{ states are feasible with } \phi_{\eta}).
\end{align*}
where $\eta$ is the hyperparameters of the safety index.

Since the probability is calculated by sampling, it is not differentiable with respect to the parameters of the safety index. Therefore we apply a derivative-free evolutionary algorithm CMA-ES\rre{~\cite{wei2022safe}}, which iteratively optimizes $\eta$ by evaluating current $\eta$ candidates and proposing new $\eta$ candidates from the best performers. 
This method works best for low-dimensionally parameterized safety index.

\section{Experiment} \label{sec: exp}

We test our Multi-Modal Robust Safe Set Algorithm (Multi-modal RSSA) on a Segway robot. We consider a tracking task with a safety specification on the tilt angle: $\phi_0 = |\varphi| - 0.1$ as shown in \cref{fig:segway}. We consider a parameterized safety index proposed by~\cite{liu2014control}: $\phi = \max\{\phi_0,-0.1^\alpha + |\varphi|^\alpha + k_v\text{sign}(\varphi)  \dot{\varphi} + \beta$\}, where $\alpha, k_v, \beta$ are learnable parameters. The nominal controller is designed to maintain $\dot{p}$ at $1m/s$. 


\begin{figure}
    \centering
    \includegraphics[width=0.5\linewidth]{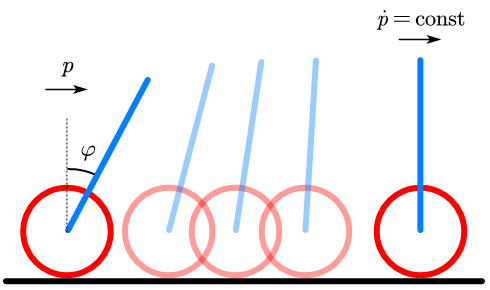}
    \caption{Segway Robot}
    \label{fig:segway}
\end{figure}

\subsection{Dynamics of the Segway Robot}

Given the wheel's position $p$ and the frame's tilt angle $\varphi$, we define $q = [p, \varphi]^T$ and $\dot{q} = [\dot{p}, \dot{\varphi}]^T$, then the state $x=[q, \dot q]^T$. Segway's dynamic model can be written as
\begin{align}
    \underbrace{\frac{d}{dt}\left[ \begin{array}{c}
        	q\\
        	\dot{q}\\
        \end{array} \right]}_{\dot x} =& \underbrace{\left[ \begin{array}{c}
        	\dot{q}\\
        	-M\left( q \right) ^{-1}H\left( q,\dot{q} \right) \\
        \end{array} \right]}_{f(x)}+\underbrace{\left[ \begin{array}{c}
        	0\\
        	M\left( q \right) ^{-1}B\\
        \end{array} \right]}_{g(x)} u \nonumber
\end{align}

where $M(q)$, $H(q, \dot{q})$, $B$ and $b_t$ are defined as follows: 
\begin{subequations}
\begin{align}
    M\left( q \right) &=\left[ \begin{matrix}
        	m_0&		mL\cos \left( \varphi \right)\\
        	mL\cos \left( \varphi \right)&		J_0\\
        \end{matrix} \right] \\
    H\left( q,\dot{q} \right) &= \left[ \begin{array}{c}
        	-mL\sin \left( \varphi \right) \dot{\varphi}^2+\frac{b_t}{R}\left( \dot{p}-R\dot{\varphi} \right)\\
        	-mgL\sin \left( \varphi \right) -b_t\left( \dot{p}-R\dot{\varphi} \right)\\
        \end{array} \right] \\
    B &= \left[ \begin{array}{c}
        	\frac{K_m}{R}\\
        	-K_m\\
        \end{array} \right]\\
    b_t &= K_m\frac{K_b}{R}
\end{align}
\end{subequations}

\subsection{Safe Control under Multi-Modal Uncertainties}

We illustrate the capability of Multi-modal RSSA on the Segway robot under both additive and multiplicative uncertainty. We compare our method with the state-of-the-art method for uni-modal uncertainty (Ellipsoid-RSSA)~\cite{wei2022persistently}. For a fair comparison, we set $\epsilon_f = 0.01$ for both methods. We use a manually designed safety index $\phi_h$ ($\alpha=1.0, k_v =1.0, \beta =0.001$). 

For additive uncertainty, we assume there is an additive noise $d$ added to $f$, that is $\dot x = f(x) + d + g(x) u$ where $d \sim \mathcal{N} (\mu_d(\theta_i), \Sigma_d(\theta_i) )$. In the Segway case, d is a 4-dimensional random variable. Its two modes are
\begin{subequations}
\begin{align}
    P(\theta_1)&=0.8\\
    \mu_d(\theta_1) &= \begin{bmatrix}
                        0.1 & -0.1 & 0.1 & -0.1
                        \end{bmatrix} \\
    \Sigma_d(\theta_1)&=
                    \begin{bmatrix}
                        0.18 & 0 & 0 & 0 \\
                        0 & 0.18 & 0 & 0.1 \\
                        0 & 0 & 0.18 & 0 \\
                        0 & 0.1 & 0 & 0.18\\
                    \end{bmatrix}
    \\
    P(\theta_2)&=0.2 \\
    \mu_d(\theta_2) &= \begin{bmatrix}
                            0.1 & -0.1 & 0.2 & -7.0
                    \end{bmatrix} \\   
    \Sigma_d(\theta_2)&=
                    \begin{bmatrix}
                        0.1 & 0 & 0 & 0 \\
                        0 & 0.1 & 0 & -0.05 \\
                        0 & 0 & 0.1 & 0 \\
                        0 & -0.05 & 0 & 0.1\\
                    \end{bmatrix}
\end{align}
\end{subequations}
As shown in \cref{fig:exp-add}, the baseline approximates the uncertainty with a uni-modal confidence bound. However, our method provides a significantly reduced confidence bound in the gradient direction ($\nabla \phi$) by considering multi-modal distribution, which enlarges the feasible safe control set.

\begin{figure}
    \centering
    \subfigure[Confidence bounds]{
        \centering
        \includegraphics[width=0.45\linewidth]{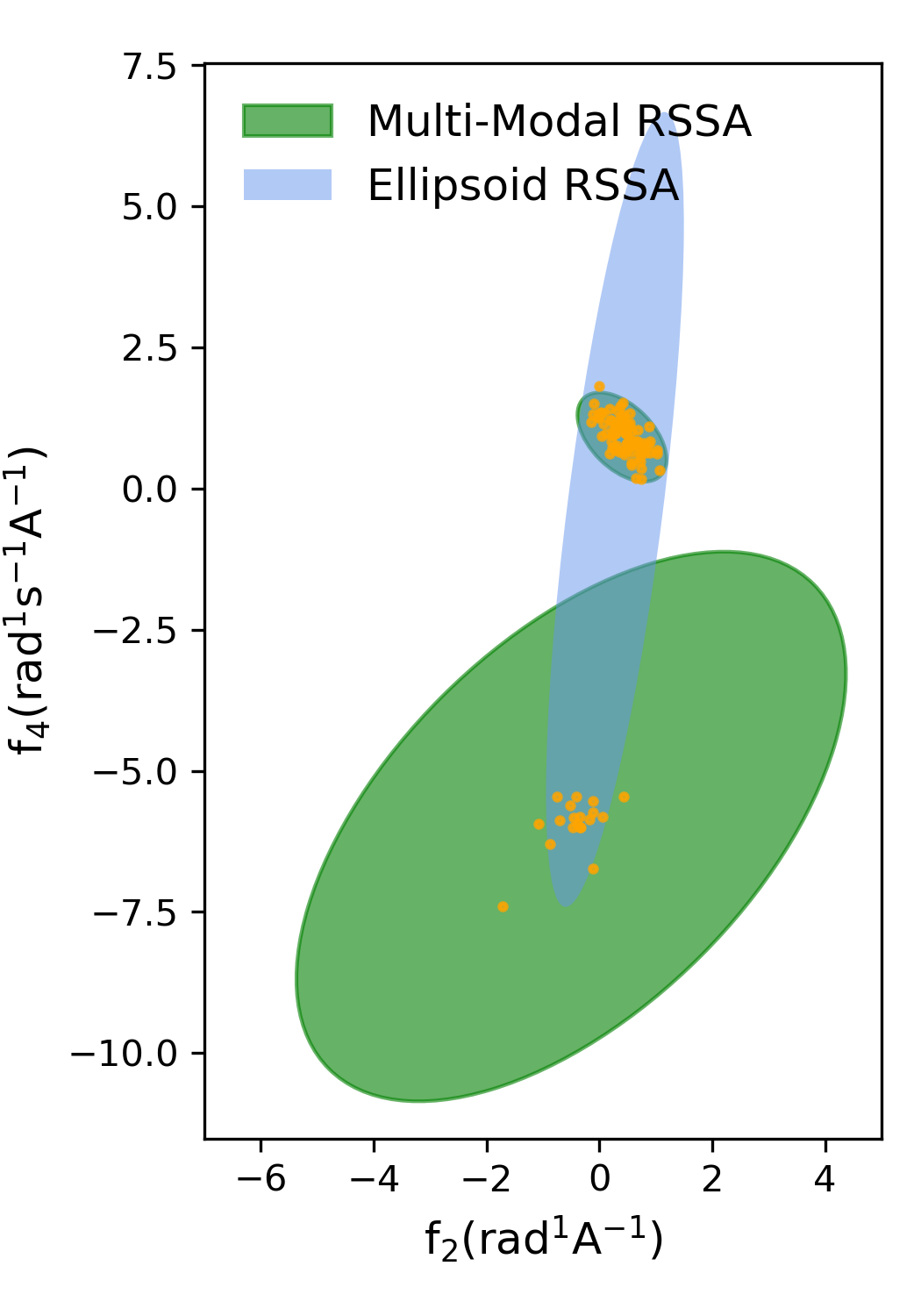}
    }
    \subfigure[Safe Control Sets]{
        \centering
        \includegraphics[width=0.45\linewidth]{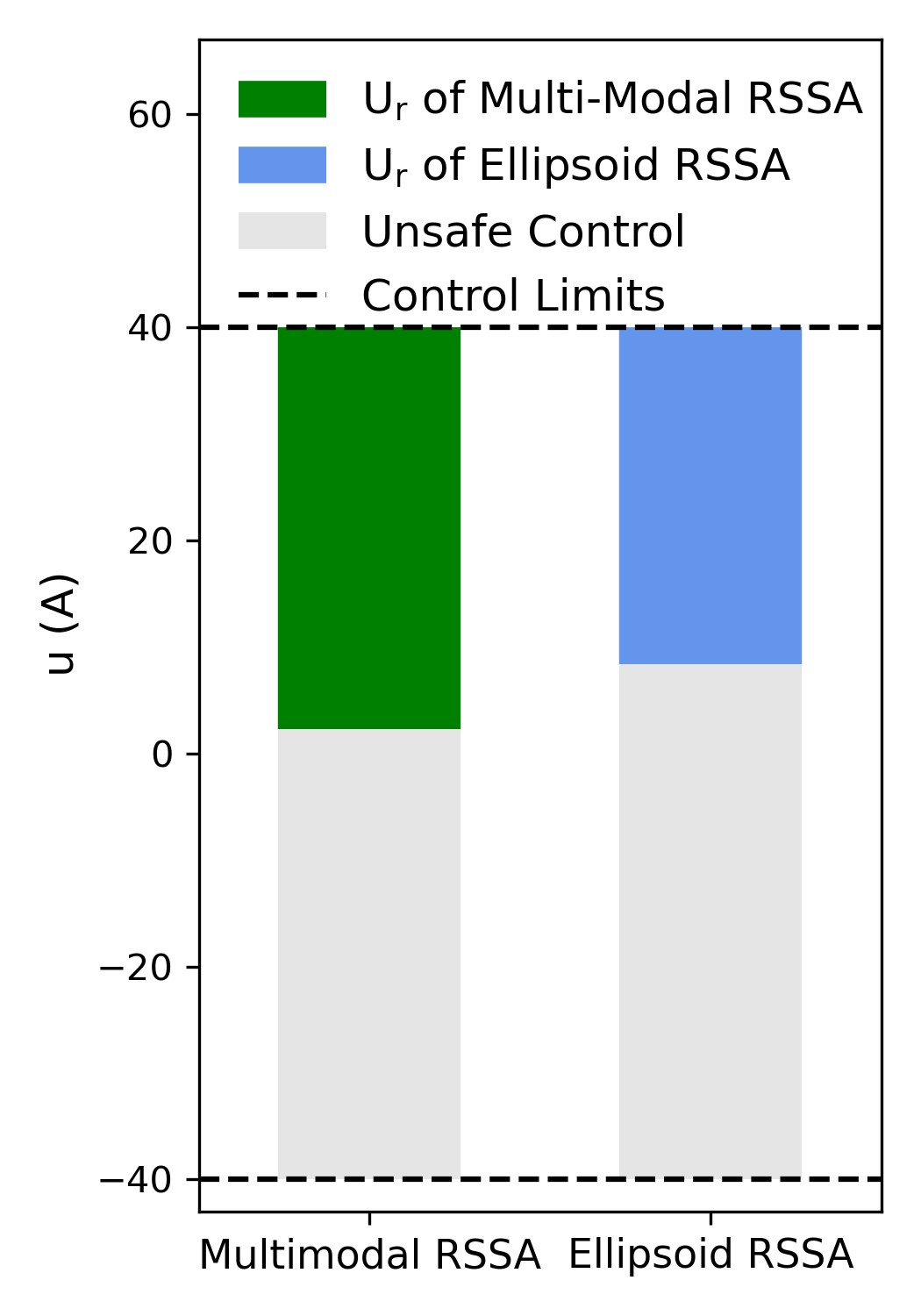}
    }
    \caption{Dynamic model of the Segway robot under additive uncertainty and the corresponding safe control sets. (a) shows confidence bounds for a 2D slice of $f(x)$ (in the second and fourth components $f_2$ and $f_4$ respectively) given by our method and the baseline. The orange points represent dynamic model samples. The ellipsoid bound of the lower mode given by Multi-Modal RSSA is large, so the upper ellipsoid bound can be smaller, which leads to a reduced uncertainty bound in the gradient direction ($\nabla \phi$) and consequently a large safe control set. (b) shows the corresponding robust safe control set $U_r$. Our method gives the largest possible $U_r$.}
    \label{fig:exp-add}
\end{figure}

For multiplicative uncertainty, we assume the motor torque constant $K_m$ follows a multi-modal Gaussian distribution: $K_m(\theta_i) \sim \mathcal{N} (\mu_k(\theta_i), \sigma_k(\theta_i)^2 )
$. Its two modes are $P(\theta_1) = 0.8$, $\mu_k(\theta_1)=2.4, \sigma_k(\theta_1) = 0.05$, and $P(\theta_2) = 0.2$, $\mu_k(\theta_2) = 4.2, \sigma_k(\theta_2) = 0.2$. After transformation, the dynamic model uncertainty follows a state-dependent multi-modal Gaussian distribution, $f(x,\theta_i)\sim \mathcal{N}(\mu_f(x,\theta_i), \Sigma_f(x,\theta_i))$ and $g(x,\theta_i)\sim \mathcal{N}(\mu_g(x,\theta_i), \Sigma_g(x,\theta_i))$.
The experiment results are shown in \cref{fig:exp-multi}. Our method provides a much tighter confidence bound of model uncertainties and achieves a much less conservative safe control than the baseline assuming uni-modal uncertainty.

\begin{figure}
    \centering
    \subfigure[Confidence bounds]{
        \centering
        \includegraphics[width=0.45\linewidth]{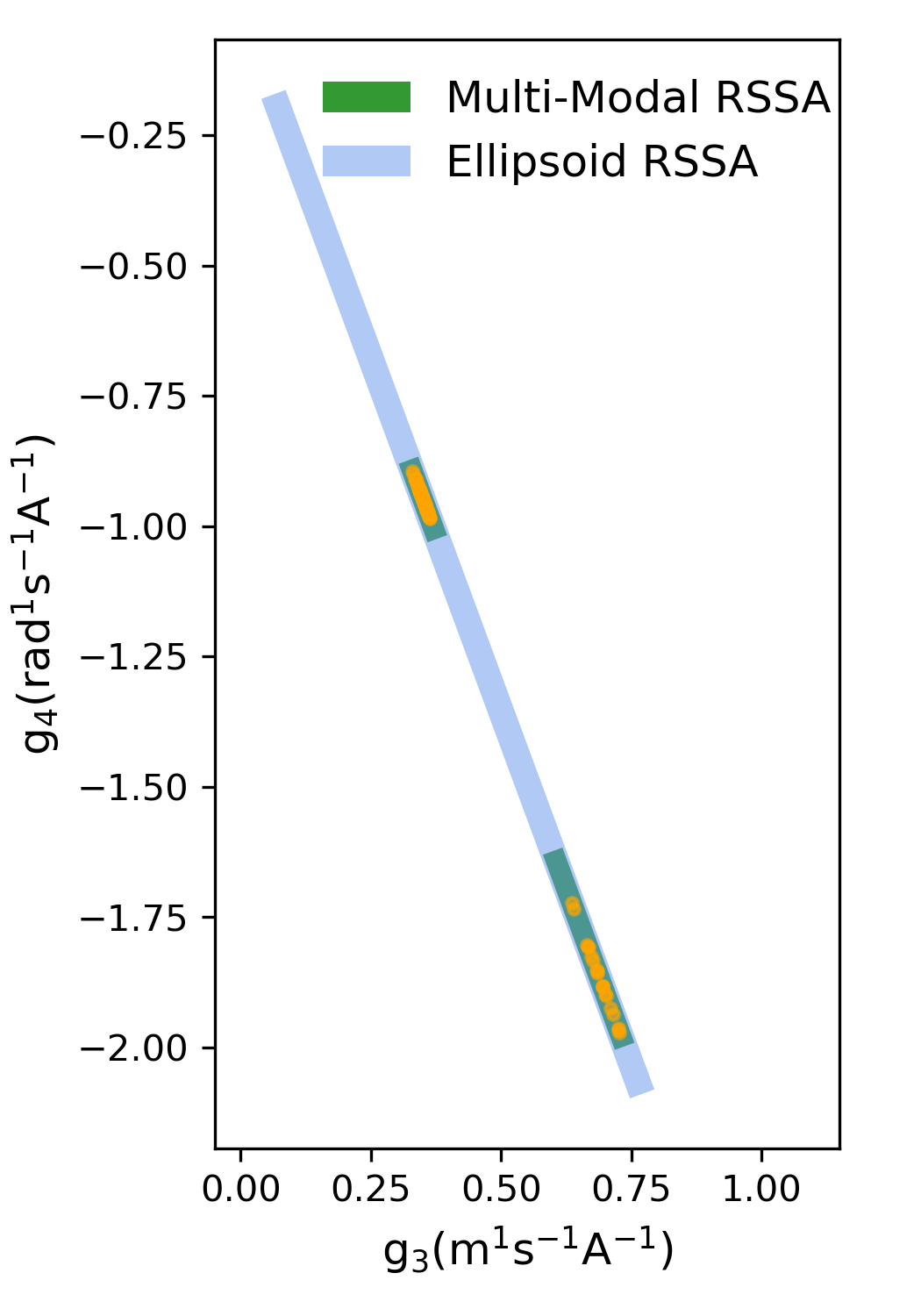}
    }
    \subfigure[Safe Control Sets]{
        \centering
        \includegraphics[width=0.45\linewidth]{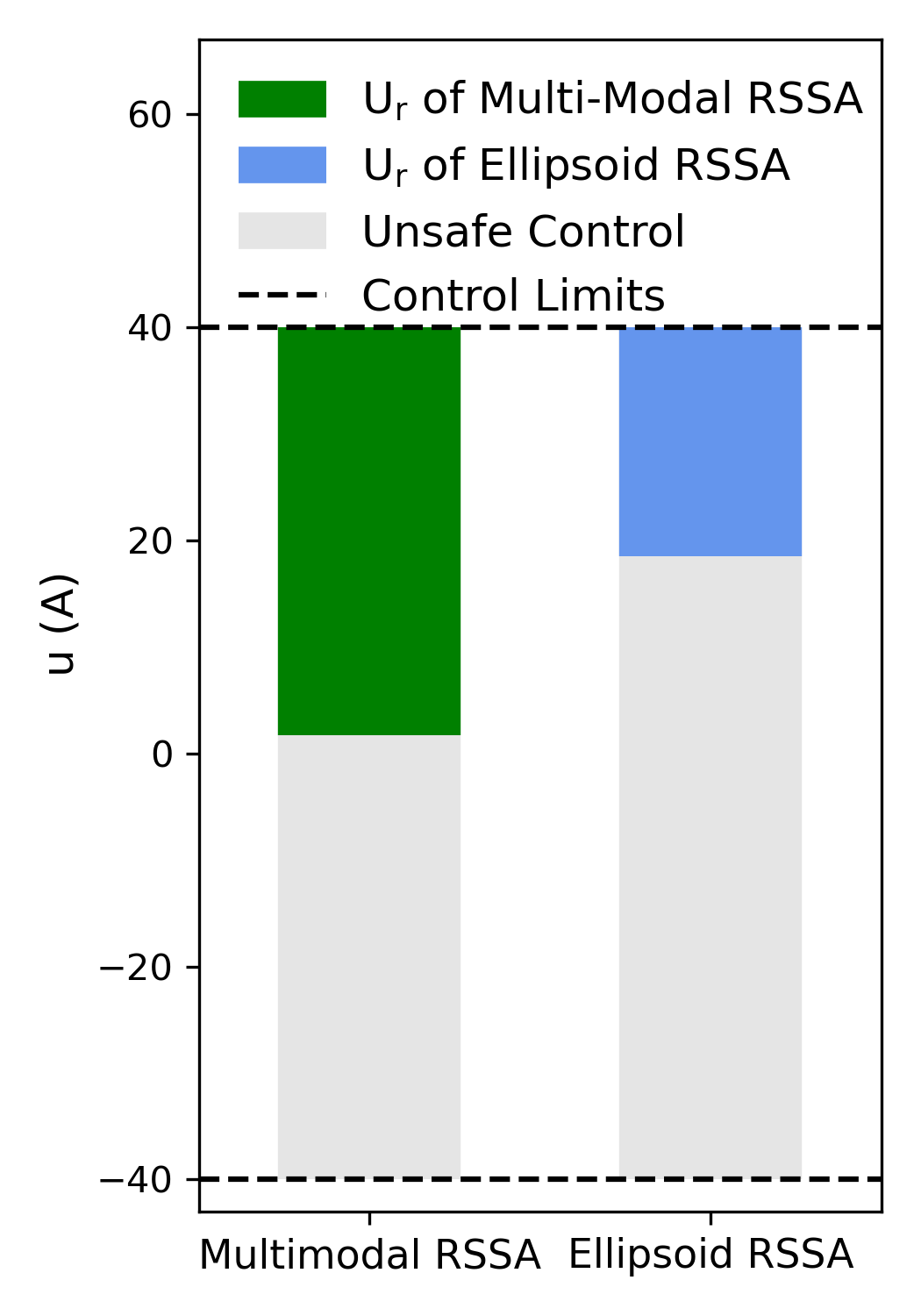}
    }
    \caption{Dynamic model of the Segway robot under multiplicative uncertainty and the corresponding safe control sets. (a) shows confidence bounds for a 2D slice of $g(x)$ (in the third and fourth components $g_3$ and $g_4$ respectively) given by our method and the baseline. The orange points represent dynamic model samples. Notably, these samples lie in a straight line due to the nature of the dynamic model. While the bounds provided by RSSA should technically be line segments without widths, we represent them as rectangle areas for clearer visualization.  Our method gives a much tighter bound than single-modal methods. (b) shows the corresponding robust safe control set $U_r$. Our method gives a much larger $U_r$.}
    \label{fig:exp-multi}
\end{figure}

\subsection{Safety Index Synthesis}

We show that our method for learning the safety index parameters ensures feasibility for the Segway robot under multiplicative uncertainty.
The search ranges for the parameters of the safety index are: $\alpha \in (0.1, 5.0) ,k_v\in ( 0.1, 5.0 ) ,\beta \in ( 0.001, 1.0 ) $. We uniformly sample $250000$ states in the whole state space. 
\Cref{fig: PR-SI-learning} compares $\phi_0$, a manually tuned safety index $\phi_h$ ($\alpha=1.0, k_v =1.0, \beta =0.001$), and a synthesized PR-SI $\phi_l$ with the Multiplicative Multi-modal RSSA solver ($\alpha=0.15, k_v =4.17, \beta = 0.55$). $\phi_l$ achieves a $0$ infeasible rate, according to which we can conclude that $P(\text{more than } 99.99\% \text{ states are feasible}) > 99.9999\%$.

\section{Discussion}

In this work, we proposed a framework for robust safe control for multi-modal uncertain dynamic models. The experiments showed that our framework can synthesize a safety index that ensures feasibility with high probability and can achieve less conservative safe control than the previous method under multi-modal uncertain dynamic models. 

Future work includes addressing non-Gaussian uncertainties, considering the correlation between $f(x, \theta_i)$ and $g(x, \theta_i)$, and varying allocations of probabilities over the confidence bounds of $f(x, \theta_i)$ and $g(x, \theta_i)$, as opposed to simply using $\sqrt{p_i}$.

\begin{figure}[tb]
\centering
\subfigure[$\phi_0$]{
\centering
\includegraphics[height=2.63cm]{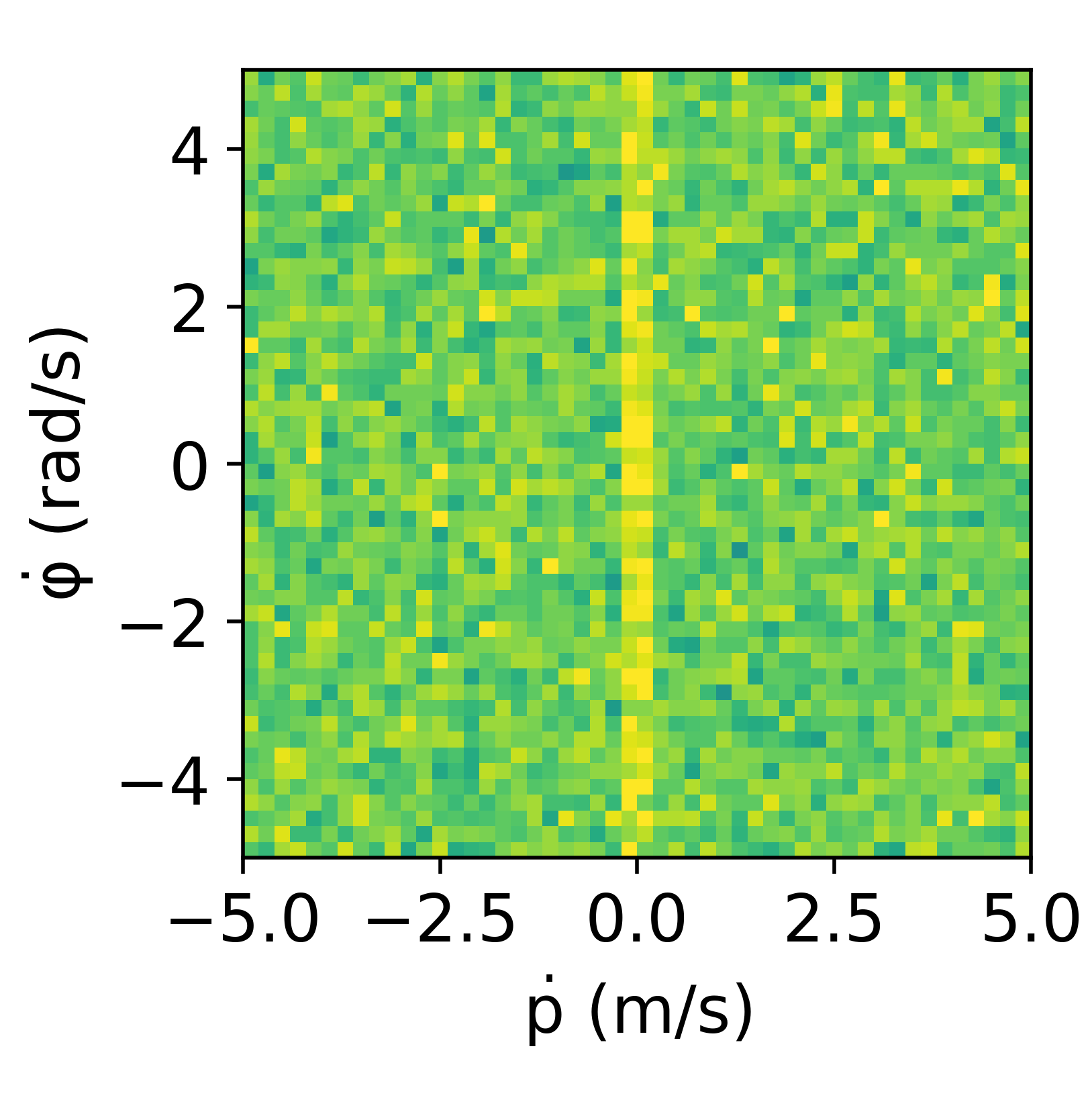}
}
\subfigure[$\phi_h$]{
\centering
\includegraphics[height=2.63cm]{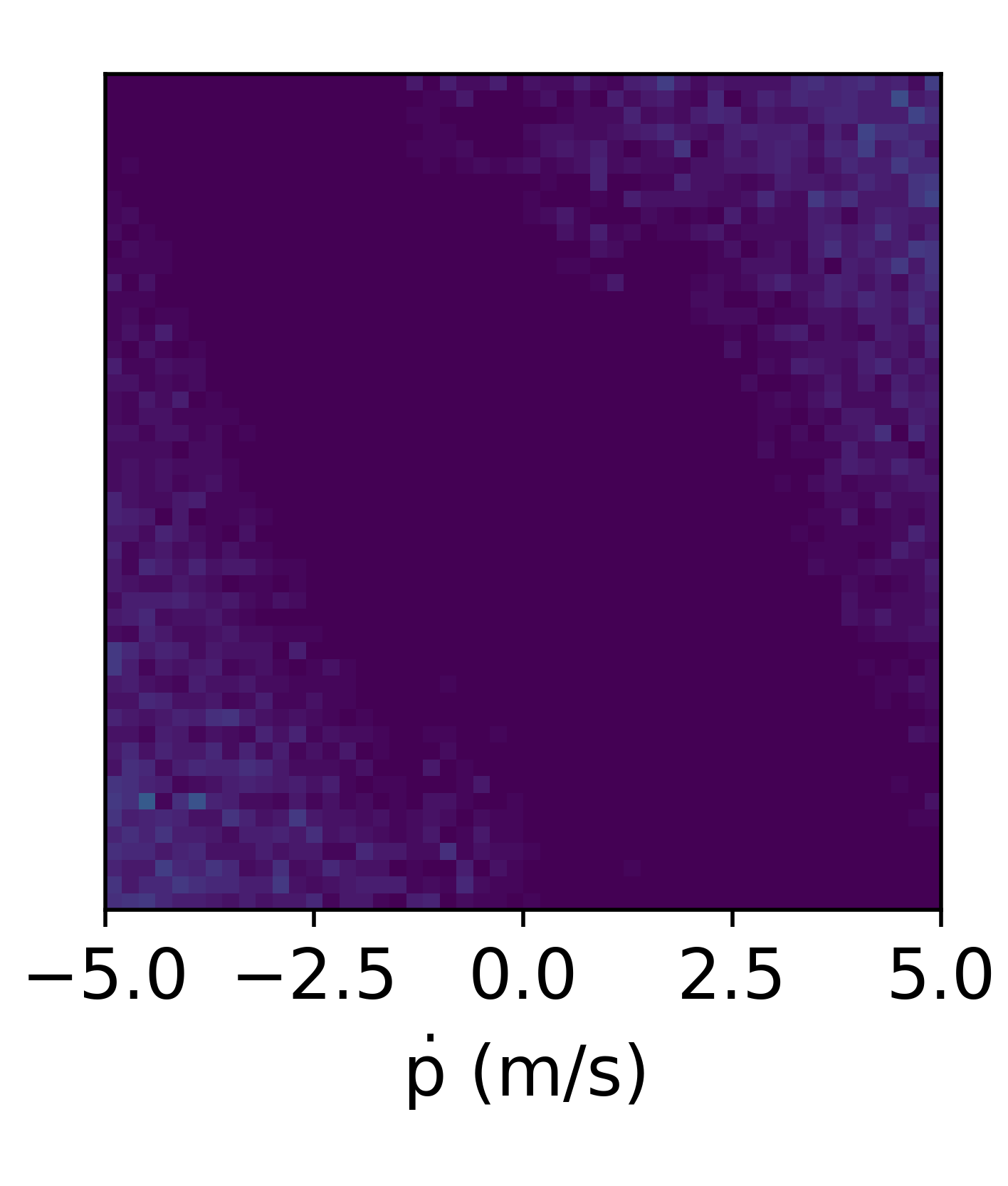}
}
\subfigure[$\phi_l$]{
\centering
\includegraphics[height=2.69cm]{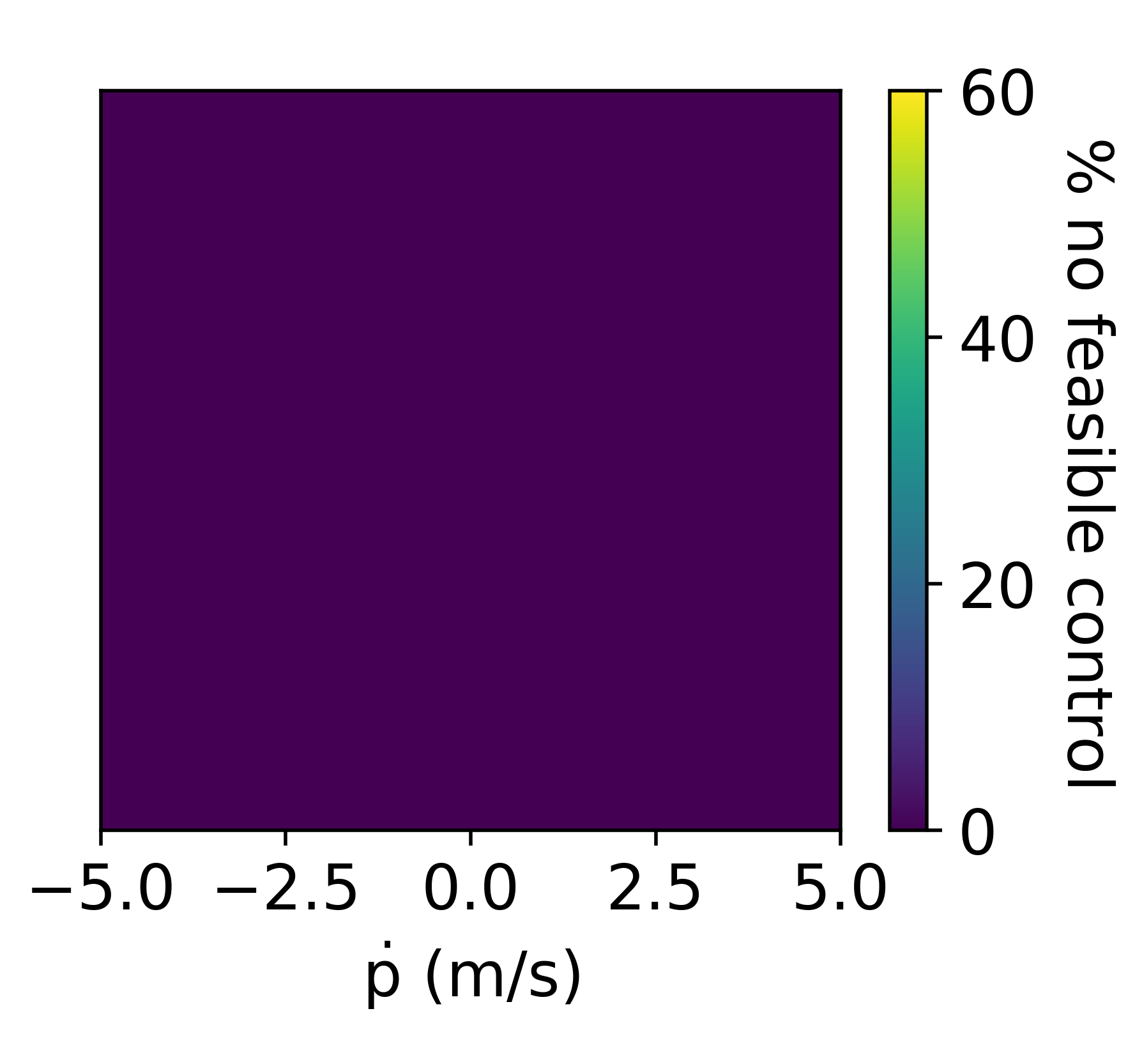}
}
\centering
\caption{\small Distribution of infeasible states in the robot configuration space of three safety indexes. Each point of the graph represents $(\dot{p}, \dot{\varphi})$. We uniformly sample 100 $\varphi$ values at each $(\dot{p}, \dot{\varphi})$ point (we do not sample $p$ because $p$ has no influence on feasibility). The color denotes how many of $\varphi$ at each point have no feasible robust safe control. (a) shows that the original safety index $\phi_0$ cannot ensure feasibility. (b) shows that a hand-designed safety index $\phi_h$ also has many infeasible states. (c) shows that our learned PR-SI $\phi_l$ ensures feasibility for all states even considering multi-modal model uncertainty.}
\label{fig: PR-SI-learning}
\end{figure}

\bibliographystyle{IEEEtran}
\bibliography{reference}



\end{document}